\documentclass{article}
\usepackage{url}
\usepackage{amsmath,epsfig,cite,multirow}
\usepackage{url}
\usepackage{amssymb,amsthm}
\usepackage{url}
\usepackage{amsmath,amssymb,amsfonts}
\usepackage{enumerate}
\usepackage{epsfig}
\usepackage{amsxtra}
\usepackage{epsf}
\usepackage{psfrag}
\usepackage{color}
\usepackage[usenames,dvipsnames]{xcolor}
\usepackage{booktabs} 

\newcommand\restr[2]{#1_{\hspace{.3ex}\rule{.01pt}{.7em}\hspace{.3ex}#2}}

\newcommand{\sparsity}{k}
\newcommand{\constraint}{\Sigma_\sparsity}

\newcommand{\nonconvex}{\constraint}

\newcommand{\proj}{\mathcal{P}}
\newcommand{\sss}{\mathcal{S}} 
\DeclareMathOperator{\supp}{supp}

\DeclareMathAlphabet{\mathbbb}{U}{bbold}{m}{n}
\newcommand{\ones}{\mathbbb 1}                
\newcommand{\order}{\mathcal{O}}

\newcommand{\Rmn}{\R^{m \times n}}

\def\S{\Sigma}

\newcommand{\obs}{\mathbf y}

\newcommand{\X}{\mathbf X}

\newcommand{\linmap}{\boldsymbol{\mathcal{A}}}

\newcommand{\numsam}{m}

\newcommand{\noise}{\boldsymbol{\varepsilon}}

\newcommand{\vectornorm}[1]{\|#1\|}
\newcommand{\Ast}{\star} 

\newcommand{\C}{\mathbb C}  

\newcommand{\bigO}{\mathcal O}

\newcommand{\abs}[1]{\left\vert#1\right\vert}

\newcommand{\observation}{{\bf w}}  
\newcommand{\signal}{\boldsymbol{\beta}}
\newcommand{\entry}{\beta}
\newcommand{\bestsignal}{\signal^\Ast}
\newcommand{\dimension}{p}

\newtheorem{problem}{Problem}
\newtheorem{remark}{Remark}
\newtheorem{proposition}{Proposition}[section]
\newtheorem{definition}{Definition}[section]

\newtheorem{theorem}{Theorem}
\newtheorem{lemma}{Lemma}
\newcommand{\bitem}{\begin{itemize}}
\newcommand{\eitem}{\end{itemize}}

\DeclareMathOperator*{\argmax}{argmax}        
\DeclareMathOperator*{\argmin}{argmin}
\DeclareMathOperator*{\minimize}{minimize}

\newcommand{\beqn}{\begin{equation}}
\newcommand{\eeqn}{\end{equation}}
\newcommand{\balign}{\begin{align}}
\newcommand{\ealign}{\end{align}}

\newcommand{\R}{ \mathbb{R} }     

\def\proj { \mathcal{P} } 

\DeclareMathOperator{\tr}{tr}

\usepackage{graphicx} 
\usepackage{subfigure} 

\usepackage{natbib}

\usepackage{algorithm}
\usepackage{algorithmic}

\usepackage{hyperref}
\hypersetup{%
bookmarks=true
}

\usepackage[accepted]{icml2013}

\icmltitlerunning{Sparse projections onto the simplex}

\begin{document} 

\twocolumn[
\icmltitle{Sparse projections onto the simplex}


\icmlauthor{Anastasios Kyrillidis}{anastasios.kyrillidis@epfl.ch}
\icmlauthor{Stephen Becker}{stephen.becker@upmc.fr}
\icmlauthor{Volkan Cevher}{volkan.cevher@epfl.ch}
\icmlauthor{Christoph Koch}{christoph.koch@epfl.ch}

\icmlkeywords{machine learning, ICML, sparsity, projections, simplex, quantum tomography}

\vskip 0.3in
]
\begin{abstract}
Most learning methods with rank or sparsity constraints use convex relaxations, which lead to optimization with the nuclear norm or the $\ell_1$-norm. However, several important learning applications cannot benefit from this approach as they feature these convex norms as constraints in addition to the non-convex rank and sparsity constraints. In this setting, we derive efficient sparse projections onto the simplex and its extension, and illustrate how to use them to solve high-dimensional learning problems in quantum tomography, sparse density estimation and portfolio selection with non-convex constraints.
\end{abstract}

\section{Introduction}
We study the following \emph{sparse} Euclidean projections:
\begin{problem} \label{problem:positive}(Simplex)
Given $\observation \in \mathbb{R}^{\dimension}$,
find a Euclidean projection of $\observation$ onto the intersection of $\sparsity$-sparse vectors  $\constraint = \big \lbrace \signal \in \mathbb{R}^{\dimension} :  \left|\lbrace i: \entry_i \neq 0 \rbrace\right|\le \sparsity  \big \rbrace$ and
the  simplex $\Delta_{\lambda}^+ = \big \lbrace \signal \in \mathbb{R}^{\dimension} : \entry_i \geq 0, ~\sum_i \entry_i = \lambda \big \rbrace$:
\begin{equation}  \label{eq:prob:2}
\mathcal{P}(\observation) \in \argmin_{\signal:\signal \in \constraint \cap \Delta_\lambda^+ } \vectornorm{\signal - \observation}_2.
\end{equation}
\end{problem}
\begin{problem} \label{problem:general}(Hyperplane)
    Replace $\Delta_{\lambda}^+$ in \eqref{eq:prob:2} with the hyperplane constraint $\Delta_{\lambda}= \big \lbrace \signal \in \mathbb{R}^{\dimension} : ~\sum_i \entry_i = \lambda \big \rbrace$.
\end{problem}

\emph{We prove that it is possible to compute such projections in quasilinear time via simple greedy algorithms.}

Our motivation with these projectors is to address important learning applications where the standard sparsity/low-rank heuristics based on the $\ell_1$/nuclear-norm are either given as a constraint or conflicts with the problem constraints. For concreteness, we highlight quantum tomography, density learning, and Markowitz portfolio design problems as running examples. We then illustrate provable non-convex solutions to minimize quadratic loss functions
 \begin{equation}\label{eq: loss function}
   f(\signal) := \|\obs-\linmap(\signal)\|^2
 \end{equation} 
subject to the constraints in Problem 1 and 2 with our projectors.  In \eqref{eq: loss function}, we assume that $\obs \in \mathbb{R}^{m}$ is given and the (known) operator $\linmap: \mathbb{R}^{\dimension}\rightarrow \mathbb{R}^{\numsam}$ is linear. 

For simplicity of analysis, our minimization approach is based on the projected gradient descent algorithm: 
\begin{equation}\label{eq: projected gradient}
  \signal^{i+1} = \mathcal{P}(\signal^i - \mu^i \nabla f(\signal^i)),
\end{equation}
where $\signal^i$ is the $i$-th iterate, $\nabla f(\cdot)$ is the gradient of the loss function, $\mu^i$ is a step-size, and $\mathcal{P}(\cdot)$ is based on Problem 1 or 2. When the linear map $\linmap$ in \eqref{eq: loss function} provides bi-Lipschitz embedding for the constraint sets, we can derive rigorous approximation guarantees for the algorithm \eqref{eq: projected gradient}; cf., \cite{garg2009gradient}.\footnote{Surprisingly, a recent analysis of this algorithm along with similar assumptions indicates that rigorous guarantees can be obtained for minimization of general loss functions other than the quadratic \cite{bahmani2011greedy}.} 

To the best of our knowledge, explicitly sparse Euclidean projections onto the simplex and hyperplane constraints have not been considered before.
The closest work to ours is the paper \cite{clash}. In \cite{clash}, the authors propose an alternating projection approach in regression where the true vector is already sparse and within a convex norm-ball constraint. In contrast, we consider the problem of projecting an {\em arbitrary} given vector onto convex-based and sparse constraints {\em jointly}. 

At the time of this submission, we become aware of \cite{pilanci2012recovery}, which considers cardinality regularized loss function minimization subject to simplex constraints. Their convexified approach relies on solving a lower-bound to the objective function and has $\bigO(p^4)$ complexity, which is not scalable. We also note that \emph{regularizing} with the cardinality constraints is generally easier: e.g., our projectors become simpler.  


\paragraph{Notation:} Plain and boldface lowercase letters represent scalars and vectors, resp.
The $i$-th entry of a vector $\observation$ is $w_i$, and $[w_i]_+=\max(w_i,0)$, while $\signal^i$ is the model estimate at the $i$-th iteration of an algorithm.
Given a set $\sss \subseteq \mathcal{N} = \{1,\ldots,\dimension \}$, the complement $\sss^c$ is defined with respect to $\mathcal{N}$, and the cardinality is $\abs{\sss}$. The support set of $ \observation $ is 
$\text{supp}(\observation) = \lbrace i: w_i \neq 0 \rbrace $. Given
a vector $ \observation \in \R^\dimension $, $ \observation_{\sss} $ is
the projection (in $\R^\dimension$) of $\observation$ onto $S$, i.e.~$\left(\observation_{\sss}\right)_{\sss^c}= 0 $,
whereas $ \restr{\observation}{\sss} \in \R^{|\sss|}$ is $\observation$ limited to $\sss$ entries.
The all-ones column vector is $\ones$, with dimensions apparent from the context. 
We define $\constraint$ as the set of all $\sparsity$-sparse subsets of $\mathcal{N}$,
and we sometimes write $\signal \in \constraint$ to mean $\supp(\signal)\in \constraint$. 
The trace of a matrix $\X$ is written $\tr(\X)$.

\section{Preliminaries}\label{sec: prelim}
\paragraph{Basic definitions:} 
Without loss of generality, assume $\observation$ is sorted in descending order, so $w_1$ is the largest element.
We denote $\proj_{\lambda^+}$ for the (convex) Euclidean projector onto the standard simplex $\Delta_{\lambda}^+$, and $\proj_{\lambda}$ for its extension to $\Delta_{\lambda}$. 
The (non-convex) Euclidean projector onto the set $\constraint$ is  $\proj_{\constraint}$, which retains the $\sparsity$-largest in magnitude elements.
In contrast to $\proj_\lambda$, the projection $\proj_{\constraint}$ need not be unique.

\begin{definition}[Operator $\proj_{L_{\sparsity}}$] \label{remark:greedyBasisAlgo}
We define $\proj_{L_{\sparsity}}(\observation)$ as the operator that keeps the $\sparsity$-largest entries of $\observation$ (not in magnitude) and sets the rest to zero. This operation can be computed in $\order(\dimension \min( \sparsity, \log(\dimension)))$-time. 
 \end{definition}

\begin{definition}[Euclidean projection $\proj_{\lambda^+}$] \label{def: simplex proj}The projector onto the simplex is given by
\begin{equation}
  \begin{split} \nonumber\label{eq:projSimplex}
    (\proj_{\lambda^+}(w) )_i = [w_i - \tau]_+,
    \,\text{where}~\tau := \frac{1}{\rho}\left( \sum_{i=1}^\rho w_i - \lambda\right )    
    \end{split}
\end{equation} for $\rho := \max\{ j : w_j > \frac{1}{j}(\sum_{i=1}^jw_i - \lambda) \}.$
\end{definition}

\begin{definition}[Euclidean projection $\proj_{\lambda}$]
    The projector onto the extended simplex is given by
\begin{equation} \nonumber \label{eq:projGeneralSimplex}
    (\proj_\lambda(w) )_i = w_i - \tau, \quad\text{where}\quad
\tau = \frac{1}{\dimension}\left( \sum_{i=1}^\dimension w_i - \lambda \right).
\end{equation}
\end{definition}

\begin{definition}[Restricted isometry property (RIP) \cite{candes06robust}] \label{def:sRIP} A linear operator $ \linmap: \mathcal{R}^{\dimension} \rightarrow \mathcal{R}^{\numsam} $ satisfies the $\sparsity$-RIP with constant $ \delta_{\sparsity}\in (0,1)$ if
\begin{equation}\label{eq:RIP}
  1-\delta_{\sparsity} \leq \vectornorm{\linmap (\signal)}_2^2/\vectornorm{\signal}_2^2 \leq 1+\delta_{\sparsity}, ~~\forall \signal \in \constraint.
\end{equation} 
\end{definition}

\paragraph{Guarantees of the gradient scheme \eqref{eq: projected gradient}:} Let $\obs = \linmap (\bestsignal) + \noise \in \mathbb{R}^\numsam$, $(\numsam \ll \dimension)$, be a generative model where $\noise$ is an additive perturbation term and $\bestsignal$ is the $k$-sparse true model generating $\obs$. If the RIP assumption \eqref{eq:RIP} is satisfied, 
then the projected gradient descent algorithm in \eqref{eq: projected gradient} features the following invariant on the objective ~\cite{garg2009gradient}:
\begin{align}
f(\signal^{i+1}) \leq \frac{2\delta_{2\sparsity}}{1-\delta_{2\sparsity} }f(\signal^i) + c_1 \vectornorm{\noise}_2,
\end{align} for $c_1 > 0$ and stepsize $\mu^i = \frac{1}{1+\delta_{2\sparsity}}$. 
Hence, for $\delta_{2\sparsity}<1/3$, the iterations of the algorithm are contractive and \eqref{eq: projected gradient} obtains a good approximation on the loss function. In addition, \cite{foucart2010sparse} shows that we can guarantee approximation on the true model via 
\begin{align}
\vectornorm{\signal^{i+1} - \bestsignal}_2 \leq 2\delta_{3\sparsity}\vectornorm{\signal^{i} - \bestsignal}_2 + c_2 \vectornorm{\noise}_2,
\end{align} for $c_1 > 0$ and $\mu^{i} = 1$. Similarly, when $\delta_{3\sparsity}<1/2$, the iterations of the algorithm are contractive.  Different step size $\mu^i$ strategies result in different guarantees; c.f., \cite{garg2009gradient, foucart2010sparse, KyrillidisCevherRecipes} for a more detailed discussion. Note that to satisfy a given RIP constant $\delta$, random matrices with sub-Gaussian entries require $m=\bigO\left(\sparsity \log(\dimension/\sparsity)/\delta^2\right)$. In low rank matrix cases, similar RIP conditions for \eqref{eq: projected gradient} can be derived with approximation guarantees; cf., \cite{SVP}. 


\section{Underlying discrete problems}
Let $\bestsignal$ be a projection of $\observation$ onto $\nonconvex \cap \Delta_\lambda^+$ or $\nonconvex \cap \Delta_\lambda$. We now make the following elementary observation:
\begin{remark} \label{remark:equiv}
The Problem 1 and 2 statements can be equivalently transformed into the following nested minimization problem: $\lbrace\sss^{\Ast}, \bestsignal\rbrace
  = $
\begin{align}\nonumber
\argmin_{\sss: \sss \in \constraint }  
\Big[ \argmin_{\substack{\signal: \signal_\sss \in \Delta_{\lambda}^+ \text{~or~} \Delta_{\lambda},\\ \signal_{\sss^c} = 0}} \vectornorm{\restr{(\signal - \observation)}{\sss}}_2^2 + \vectornorm{\restr{\observation}{\sss^c}}_2^2\Big],
\end{align}
where $\supp(\bestsignal) = \sss^{\Ast}$ and $\bestsignal \in \Delta_{\lambda}^+$ or $\Delta_{\lambda}$.
\end{remark}
Therefore, given  $\sss^\Ast=\supp(\bestsignal)$, we can find $\bestsignal$ by projecting $\observation_{\sss^\Ast}$ onto
$\Delta_{\lambda}^+$ or $\Delta_{\lambda}$ within the $\sparsity$-dimensional space. Thus, the difficulty is finding $\sss^\Ast$. Hence, we split the problem into the task of finding the support and then finding the values on the support. 

\subsection{Problem 1}
Given any support $\sss$, the unique corresponding estimator is $\restr{\widehat{\signal}}{\sss} = \proj_{\lambda^+}(\restr{\observation}{\sss})$. We conclude that $\bestsignal$ satisfies ${\bestsignal}_{(\sss^{\Ast})^c} = 0$ and
$    \restr{\bestsignal}{\sss^\Ast} = \proj_{\lambda^+}(\restr{\observation}{\sss^\Ast})$, 
and 
\begin{align}  \label{eq:modularAA} 
    \sss^{\Ast} &\in \argmin_{\sss: \sss \in \constraint} \vectornorm{\proj_{\lambda^+}(\restr{\observation}{\sss}) - \restr{\observation}{\sss}}^2_2 + \vectornorm{\restr{\observation}{\sss^c}}_2^2 \nonumber\\ 
        &= \argmax_{\sss: \sss \in \constraint} F_+(\sss)
\end{align}
where $F_+(\sss) := \sum_{i \in \sss} \left( w_i^2 - ((\proj_{\lambda^+}(\restr{\observation}{\sss}))_i - w_i)^2\right)$. 

This set function can be simplified to
\begin{equation}\label{eq:setplus}
  F_+(\sss) = \sum_{i\in\sss} (w_i^2 - \tau^2),
\end{equation} 
where $\tau$ (which depends on $S$) is as in Lemma \ref{biggerThanTau}.

\begin{lemma} \label{biggerThanTau}
Let $\signal = \proj_{\lambda^+}(\observation)$ where $\beta_i = [w_i-\tau]_+$.
Then, $w_i \ge \tau$ for all $i\in \sss=\supp(\signal)$.
Furthermore, $\tau = \frac{1}{|\sss|}\left( \sum_{i\in \sss} w_i - \lambda \right)$.
\end{lemma} 
\begin{proof} Directly from the definition of $\tau$ in Definition \ref{def: simplex proj}. The intuition is quite simple: the ``threshold'' $\tau$ should be smaller than the smallest entry in the selected support, or we unnecessarily shrink the coefficients that are larger without introducing any new support to the solution. Same arguments apply to inflating the coefficients to meet the simplex budget.\end{proof}
\subsection{Problem 2}
Similar to above, we conclude that $\bestsignal$ satisfies
$    \restr{\bestsignal}{\sss^{\Ast}} = \proj_{\lambda}(\restr{\observation}{\sss^{\Ast}})$ and $\restr{\bestsignal}{(\sss^{\Ast})^c} = 0$, where
\begin{align}  \label{eq:modularA}
    \sss^{\Ast} &\in \argmin_{\sss: \sss \in \constraint} \vectornorm{ {\bf z} - \observation}_2 = \argmax_{\sss: \sss \in \constraint} F(\sss)
\end{align} where ${\bf z} \in \mathbb{R}^p$ with $\restr{{\bf z}}{\sss} = \proj_{\lambda}(\restr{\observation}{\sss})$ and $\restr{{\bf z}}{\sss^c} = 0$ and 
$F(\sss) := \left( \sum_{i \in \sss}  w_i^2 \right) - \frac{1}{|S|}(\sum_{i \in \sss}  w_i-\lambda)^2$. 

\section{Sparse projections onto \texorpdfstring{$\Delta_\lambda^+$ and $\Delta_\lambda$}{Delta+ and Delta}}\label{sec: algos}
Algorithm~\ref{alg:1} below suggests an obvious greedy approach for the projection onto $\nonconvex \cap \Delta_\lambda^+$. We select the set $\sss^\Ast$ by naively projecting $\observation$ as $\proj_{L_{\sparsity}}(\observation)$. Remarkably, this gives the correct support set for Problem~\ref{problem:positive}, as we prove in Section \ref{sec:positivesimplex}. We call this algorithm the greedy selector and simplex projector  (GSSP). The overall complexity of GSSP is dominated by the sort operation in $p$-dimensions.

\begin{algorithm}[t]
    \caption{GSSP}
    \label{alg:1}
    \begin{algorithmic}[1]
        \STATE $\sss^\Ast = \supp(\proj_{L_{\sparsity}}(\observation)) $ \hfill \COMMENT{Select support}
        \STATE $\restr{\signal}{\sss^\Ast} = \proj_{\lambda^+}( \restr{\observation}{\sss^\Ast}), \restr{\signal}{(\sss^{\Ast})^c} = 0 $  \hfill \COMMENT{Final projection}
    \end{algorithmic}
\end{algorithm}
\begin{algorithm}[t] 
    \caption{GSHP}
    \label{alg:2}
    \begin{algorithmic}[1]
        \STATE $\ell=1$ , $\sss = j, \quad j \in \arg\max_i \left[\lambda w_i\right]$\hfill \COMMENT{Initialize}
        \STATE Repeat: ${\ell \leftarrow \ell+1}$, ${\sss\leftarrow \sss \cup j}$, where \\ ${\qquad j\in \arg\max_{i\in \mathcal{N}\setminus \mathcal{S}} \left|w_i - \frac{\sum_{j \in \sss}w_j-\lambda}{\ell-1}\right|}$   \hfill \COMMENT{Grow}
        \STATE Until $\ell= k$, set $\sss^\Ast\leftarrow \sss$ \hfill \COMMENT{Terminate}
        \STATE $\restr{\signal}{\sss^\Ast} = \proj_{\lambda}( \restr{\observation}{\sss^\Ast}),
        \;  \restr{\signal}{(\sss^{\Ast})^c} = 0 $ \hfill \COMMENT{Final projection}
    \end{algorithmic}
\end{algorithm}

Unfortunately, the GSSP fails for Problem~\ref{problem:general}.
As a result, we propose Algorithm~\ref{alg:2} for the $\nonconvex \cap \Delta_\lambda$ case which is non-obvious. The algorithm first selects the index of the largest element that has the same sign as $\lambda$. It then grows the index set one at a time by finding the farthest element from the current mean, as adjusted by lambda. Surprisingly, the algorithm finds the correct support set, as we prove in  Section \ref{sec: hyperplane}. We call this algorithm the greedy selector and hyperplane projector (GSHP), whose overall complexity is similar to GSSP.

\vspace{-3mm}
\section{Main results}
\begin{remark}\label{remark:|S|=s} 
    When the symbol $\sss$ is used as $\sss=\supp(\bar{\signal})$ 
    for any $\bar{\signal}$, 
    then if $|\sss| < \sparsity$, we enlarge $\sss$ until
    it has $\sparsity$ elements by taking the first $\sparsity - |\sss|$ elements
    that are not already in $\sss$, and setting $\bar{\signal}=0$ on these elements. The lexicographic approach is used to break ties when there are multiple solutions.
\end{remark}
\vspace{-1mm}
\vspace{-2mm}
\subsection{Correctness of GSSP} \label{sec:positivesimplex}

\begin{theorem}\label{thm:simplexStandard}
    Algorithm~\ref{alg:1} exactly solves Problem~\ref{problem:positive}.
\end{theorem}
\vspace{-5mm}
\begin{proof}
Intuitively, the $k$-largest coordinates should be in the solution. To see this, suppose that ${\bf u}$ is the projection of $\observation$. Let $w_i$ be one of the $k$-(most positive) coordinates of $\observation$ and $u_i = 0$. Also, let $w_j < w_i,~i \neq j$ such that $u_j > 0$. We can then construct a new vector ${\bf u}'$ where $u'_j = u_i = 0$ and $u'_i = u_j$. Therefore, ${\bf u}' $ satisfies the constraints, and it is closer to $\observation$, i.e., $\|\observation-{\bf u}\|^2_2 - \|\observation-{\bf u}'\|^2_2 = 2u_j(w_i - w_j) > 0$. Hence, ${\bf u}$ cannot be the projection.  

\newcommand{\ssg}{{\hat{\sss}}}
\newcommand{\bfu}{{\bf u}}
\newcommand{\bfuh}{{\bf \hat{u}}}

To be complete in the proof, we also need to show that the cardinality $k$ solutions are as good as any other solution with cardinality less than $k$.
Suppose there exists a solution $\bfu$ with support $|\sss| < k$. Now add \emph{any} elements to $\sss$ to form $\ssg$ with size $k$. Then consider $\observation$ restricted to $\ssg$, and let $\bfuh$ be its projection onto the simplex. 
Because this is a projection, 
$ \| \restr{\bfuh}{\ssg}  - \restr{\observation}{\ssg} \| \le 
\| \restr{\bfu}{\ssg}   - \restr{\observation}{\ssg} \| $, 
hence $\| \bfuh - \observation\| \le \|\bfu-\observation\|$.
\end{proof}

\subsection{Correctness of GSHP} \label{sec: hyperplane}
\begin{theorem}\label{thm: GSHP}
    Algorithm~\ref{alg:2} exactly solves Problem~\ref{problem:general}.
\end{theorem}
\vspace{-4mm}

\def\S{{\cal S}}
\def\Smep{\S \backslash \{ e' \}}
\def\avg{\mathrm{avg}}
\vspace{-1mm}
\begin{proof}
To motivate the support selection of GSHP, we now identify a key relation that holds for any ${\bf b}\in \mathbb{R}^\sparsity$:
\begin{equation}
 \sum_{i=1}^\sparsity b_i^2 - \frac{\left( \sum_{i=1}^\sparsity b_i -\lambda\right)^2} {\sparsity} = \qquad \qquad\nonumber
\end{equation}
\begin{equation}\label{eq:differences} 
\lambda(2b_1-\lambda) +
\sum_{j=2}^\sparsity \frac{j-1}{j}\Big(b_j - \frac{\sum_{i=1}^{j-1} b_i-\lambda}{j-1} \Big)^2.
\end{equation}
By its left-hand side, this relation is invariant under permutation of ${\bf b}$. Moreover, the summands in the sum over $k$ are certainly non-negative for $\sparsity\ge 2$, so without loss of generality the solution sparsity of the original problem is  $||\bestsignal||_0 = \sparsity$.
For $\sparsity=1$, $F$ is maximized by picking an index $i$ that maximizes $\lambda w_i$, which is what the algorithm does.

For the sake of clarity (and space), we first describe the proof of the case $k \ge 2$
for $\lambda=0$ and then explain how it generalizes for $\lambda\ne 0$. In the sequel, let us use the shortcut $\avg(S) = \frac{1}{|S|} \sum_{j \in S} w_j$. 

Let $\S$ be an optimal solution index set and
let $I$ be the result computed by the algorithm.
For a proof (of the case $\sparsity \ge 2, \lambda = 0$) by contradiction, assume that $I$ and $\S$ differ.
Let $e$ be the first element of $I \backslash \S$
in the order of insertion into $I$ by the algorithm.
Let $e'$ be the element of
$\S \backslash I_0$ that lies closest to $e$.
Without loss of generality, we may assume that $w_e \neq w_{e'}$, otherwise we could have chosen $\Smep \cup \{e\}$
rather than $\S$ as solution in the first place.
Let $I_0 \subseteq I \cap \S$ be the indices added to $I$ by the algorithm
before $e$.
Assume that $I_0$ is nonempty. We will later see how to ensure this.

Let $a := \avg(I_0)$ and $a' := \avg(\Smep)$.
There are three ways in which $w_e$, $w_{e'}$ and $a'$ can be ordered
relative to each other:
\begin{enumerate}[1.]
\item
$e'$ lies between $e$ and $a'$, thus
$|w_{e'} - a'| < |w_e - a'|$
since $w_e \neq w_{e'}$.

\item
$a'$ lies between $e$ and $e'$.
But then, since there are no elements of $\S$ between $e$ and $e'$,
$\S \backslash I_0$ moves the average $a'$ beyond $a$ away from $e$ towards $e'$, so
$|w_{e'} - a'| < |w_{e'} - a|$ and
$|w_e - a| < |w_e - a'|$.
But we know that
$|w_{e'} - a| < |w_e - a|$ since
$e = \mbox{argmax}_{i \in I_0} |w_i - a|$
by the choice of the greedy algorithm and $w_e \neq w_{e'}$. Thus
$|w_{e'} - a'| < |w_e - a'|$.

\item
$|w_e - a'| < |w_{e'} - a'|$, i.e., 
$e$ lies between $a'$ and $e'$.
But this case is impossible: compared to $a$, $a'$ averages over additional values 
that are closer to $a$ than $e$, and $e'$ is one of them.
So $a'$ must be on the same side as $e'$ relative to $e$, not the opposite side.
\end{enumerate}

So
$|w_{e'} - a'| < |w_e - a'|$ is assured in all cases. Note in particular that if
$|S| \ge 1$, $|w_i - \avg(S)| ~\theta ~|w_j - \avg(S)|$,
then
\begin{align}
F(S \cup \{ i \}) &=
F(S) + \frac{\sparsity-1}{\sparsity} \Big( w_i - \avg(S) \Big)^2 \nonumber
\\
&~~\theta~~
F(S) + \frac{\sparsity-1}{\sparsity} \Big( w_j - \avg(S) \Big)^2 \nonumber
\\
&= F(S \cup \{j\}),
\label{eq:replacement}
\end{align}
where $\theta$ is either `$=$' or `$<$'. By inequality~(\ref{eq:replacement}),
$F(\S) < F((\Smep) \cup \{e\})$.
But this means that $\S$ is not a solution: contradiction.

We have assumed that $I_0$ is nonempty; this is ensured because
any solution $\S$ must contain at least an index $i \in \argmax_j w_j$.
Otherwise, we could replace a maximal index w.r.t.\ $w$ in $\S$ by this $i$
and get, by \eqref{eq:replacement}, a larger $F$ value.
This would be a contradiction with our assumption that $\S$ is a solution.
Note that this maximal index is also picked (first) by the algorithm.
This completes the proof for the case $\lambda = 0$.
Let us now consider the general case where $\lambda$ is unrestricted.

We reduce the general problem to the case that $\lambda=0$.
Let us write $F_{w, \lambda}$ to make the parameters $w$ and $\lambda$ explicit when talking of $F$.
Let $w'_{i^*} := w_{i^*} - \lambda$ for one $i^*$ for which
$\lambda w_{i^*}$ is maximal,
and let $w'_i := w_i$ for all other $i$.
We use the fact that, by the definition of $F$,
\[
F_{w, \lambda}(\S) = 2\lambda w'_{i^*} + \lambda^2 + F_{w', 0}(\S)
\]
when $\S$ contains such an element $i^* \in \mbox{argmax}_j (\lambda w_j)$.
Clearly, $i^*$ is an extremal element w.r.t.\ $w$ and $w_{i^*}$ has maximum distance from $-\lambda$, so
\[
i^* \in \argmax_j \left| w_{j} - \frac{\sum_{i \neq j} w_i - \lambda}{j-1} \right|.
\]
By (\ref{eq:differences}), $i^*$
must be in the optimal solution for $F_{w, \lambda}$.
Also, $F_{w', 0}(\S)$ and $2\lambda w'_{i^*} + \lambda^2 + F_{w', 0}(\S)$ are maximized by the same index sets $\S$
when $i^* \in \S$ is required. 
Thus,
\[
\argmax_\S F_{w, \lambda}(\S) = \argmax_{\S: j \in \S} F_{w', 0}(\S).
\]
Now observe that our previous proof for the case $\lambda = 0$ also works if one adds a constraint that one or more
indices be part of the solution: If the algorithm computes these elements as part of its result $I$,
they are in $I_0 = I \cap \S$.
But this is what the algorithm does on input $(w, \lambda)$; it chooses $i^*$ in its first step and then proceeds
as if maximizing $F_{w', 0}$.
Thus we have established the algorithm's correctness.
\end{proof}

\section{Application: Quantum tomography}
\paragraph{Problem:} In quantum tomography (QT), we aim to learn a \emph{density matrix} $\X^{\star} \in \mathbb{C}^{d\times d}$, which is Hermitian (i.e., $(\X^\star)^H = \X^{\star}$), positive semi-definite (i.e., $\X^{\star} \succeq 0$) and has $\text{rank}(\X^{\star}) = r$ and $\tr(\X^{\star}) = 1$.  The QT measurements are $\obs = \linmap(\X^{\star}) +\boldsymbol{\eta}$, where $(\linmap(\X^{\star}))_i = \tr({\bf E}_i \X^{\star}) + \eta_i$, and $\eta_i$ is  zero-mean Gaussian. The operators ${\bf E}_i$'s are the tensor product of the $2\times 2$ Pauli matrices \cite{liu2011universal}. 

Recently, \cite{liu2011universal} showed that 
almost all such tensor constructions of $\bigO(rd \log^6 d)$ Pauli measurements satisfy the so-called rank-$r$ restricted isometry property (RIP) for all $\X\in \left\{\X\in\C^{d\times d} :  \X \succeq 0, \text{rank}(\X) \leq r, \|\X\|_*\le \sqrt{r} \|\X\|_F\right\}$: 
\begin{equation}\label{eq: RIP-Pauli}
\left(1-\delta_r\right)\|\X\|_F^2 \leq \|\linmap(\X)\|_F^2 \leq \left(1+\delta_r\right)\|\X\|_F^2,
\end{equation}
where $\|\cdot\|_*$ is the nuclear norm (i.e., the sum of singular values), which reduces to $\tr(\X)$ since $\X \succeq 0$. This key observation enables us to leverage the recent theoretical and algorithmic advances in low-rank matrix recovery from a few affine measurements. 

The standard matrix-completion based approach to recover $\X^{\star}$ from $\obs$ is the following convex relaxation:
\begin{equation} \label{eq:conventional}
    \minimize_{ \X \succeq 0} \|\linmap(\X) - \obs\|_F^2 + \lambda \|\X\|_*.
\end{equation}
This convex approach is both tractable and amenable to theoretical analysis~\cite{QuantumTomoPRL,liu2011universal}. 
As a result, we can provably reduce the number of samples $\numsam$  from $\bigO(d^2)$ 
to $ \tilde{\mathcal O}(rd)$~\cite{liu2011universal}. 

Unfortunately, this convex approach fails to account for the physical constraint $\|\X\|_*=1$. To overcome this difficulty, the relaxation parameter $\lambda$ is tuned to obtain solutions with the desired rank followed by normalization to heuristically meet the trace constraint. 

In this section, we demonstrate that one can do significantly better via the non-convex algorithm based on \eqref{eq: projected gradient}. A key ingredient then is the following projection:  
\begin{equation}\label{eq:rankProj}
\widehat{\bf B} \in \argmin_{{\bf B} \succeq 0} \vectornorm{{\bf B} - {\bf W}}_F^2~\text{s.t.}~\text{rank}({\bf B}) = r,~\text{tr}({\bf B}) = 1, 
\end{equation}
for a given \emph{Hermitian} matrix ${\bf W} \in \mathbb{R}^{n \times n}$. Since the RIP assumption holds here, we can obtain rigorous guarantees based on a similar analysis to \cite{garg2009gradient,foucart2010sparse, SVP}.

To obtain the solution, we compute the eigenvalue decomposition ${\bf W}={\bf U}{\Lambda}_{\bf W}{\bf U}^H$ and then use the unitary invariance of the problem to  solve $ {\bf D}^\star \in \argmin_{\bf D} \|{\bf D}-{\bf \Lambda}_{\bf W}\|_F$  subject to $\|{\bf D}\|_* \le 1$ and $\text{rank}({\bf D}) \le r$, and from ${\bf D}^\star$ form ${\bf U}{\bf D}^\star {\bf U}^H$ to obtain a solution.
In fact, we can constrain ${\bf D}$ to be diagonal, and thus reduce the matrix problem  to the vector version for ${\bf D} = \text{diag}({\bf d})$, where the projector in Problem \ref{problem:positive} applies. This reduction follows from the  well-known result:

\begin{proposition}[\cite{mirsky1960symmetric}] \label{prop:1}
Let ${\bf A},{\bf B} \in \Rmn$ and $q=\min\{m,n\}$. Let $\sigma_i({\bf A})$ be the singular values of ${\bf A}$ in descending order (similarly for ${\bf B}$). Then,
$$ \sum_{i=1}^q \bigl(\sigma_i({\bf A}) - \sigma_i({\bf B})\bigr)^2  \le \|{\bf A}-{\bf B}\|_F^2. $$
\end{proposition}
Equation~\eqref{eq:rankProj} has a solution if $r\ge 1$ since the constraint set is non-empty and compact (Weierstrass's theorem).
As the vector reduction achieves the lower bound, it is an optimal projection. 

\begin{figure}[t] 
    \centering
    \includegraphics[width=.95\linewidth]{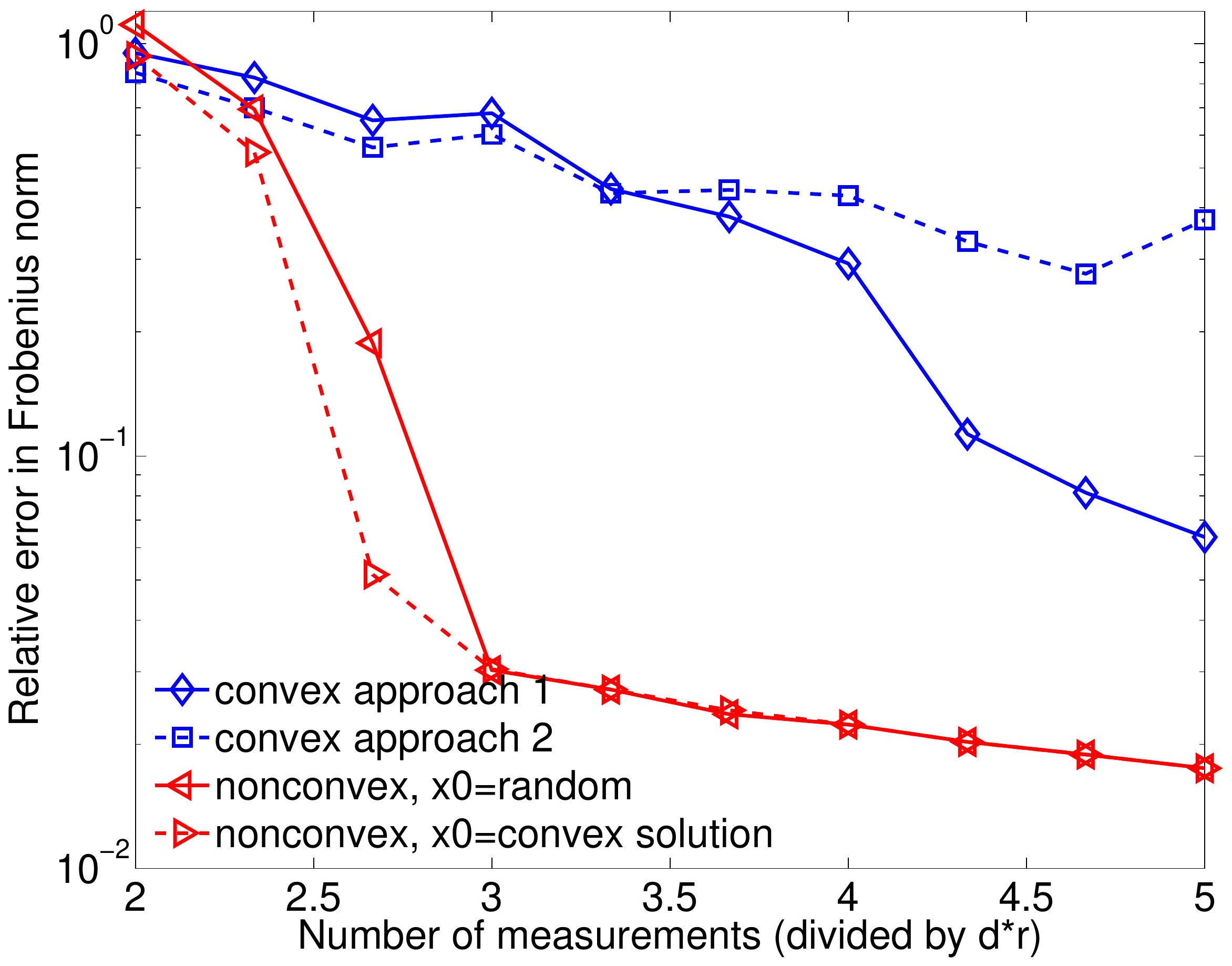}
    \caption{Quantum tomography with $8$ qubits and $30$~dB SNR: Each point is the median over 10 random realizations.   Convex approach $1$ refers to~\eqref{eq:conventional} and approach $2$ is~\eqref{eq:conventional_2}.}
    \label{fig:quantum1}
\end{figure}
\begin{figure}[t] 
    \centering
    \includegraphics[width=.95\linewidth]{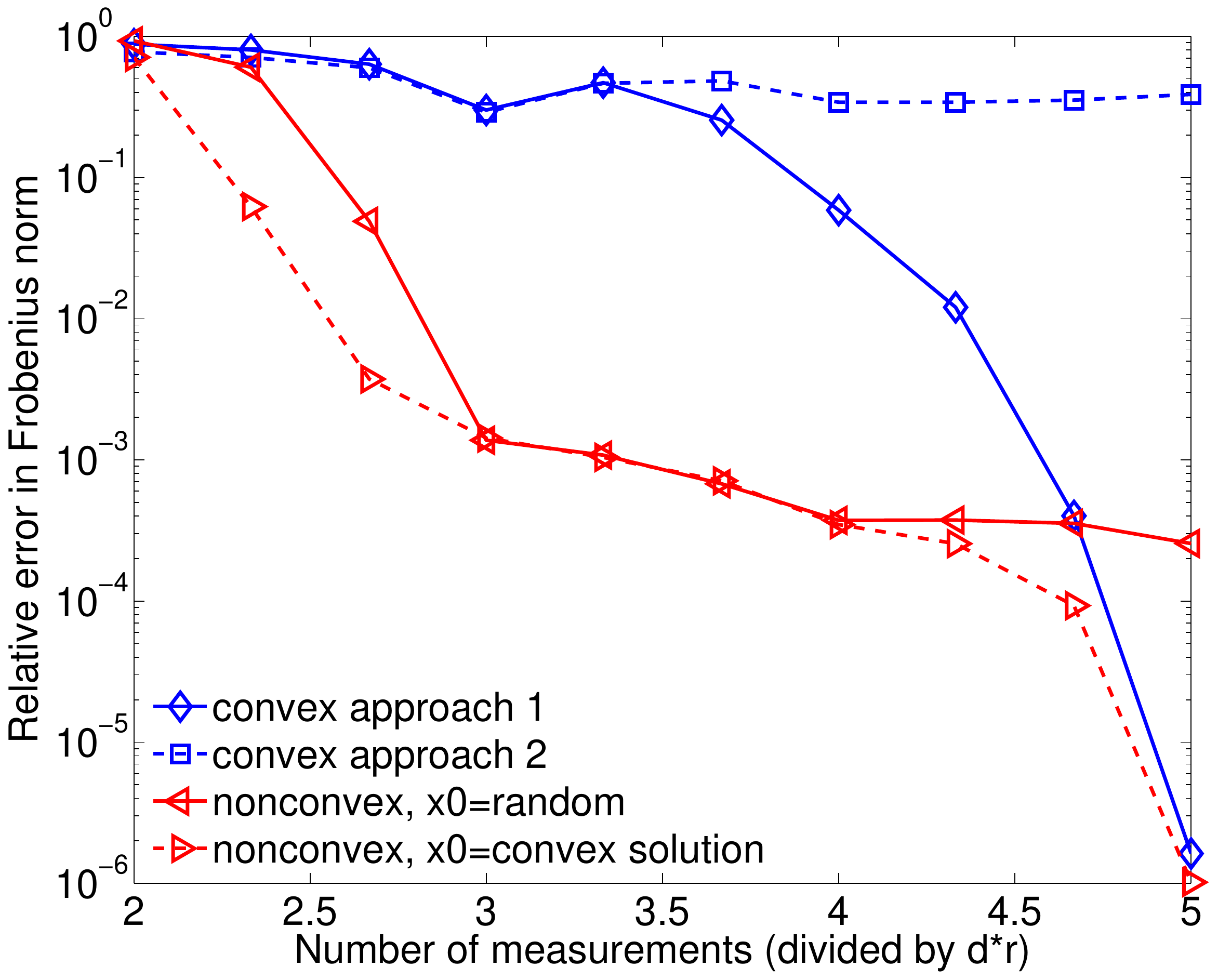}
    \caption{Same as Figure~\ref{fig:quantum1} but with $7$ qubits, no noise. }
    \label{fig:quantum2}
\end{figure}
\paragraph{Numerical experiments:}
We numerically demonstrate that the ability to project onto trace and rank constraints jointly can radically improve recovery even with the simple gradient descent algorithm as in \eqref{eq: projected gradient}. We follow the same approach in \cite{QuantumTomoPRL}: we generate random Pauli measurements and add white noise. The experiments that follow use a $8$ qubit system ($d=2^8$) with noise SNR at $30$~dB (so the absolute noise level changes depending on the number of measurements), and a $7$ qubit noiseless system.

The measurements are generated using a random real-valued matrix $\X^\star$ with rank $2$, although the algorithms also work with complex-valued matrices. A $d\times d$ rank $r$ real-valued matrix has $dr-r(r-1)/2 \approx dr$ degrees of freedom, hence we need at least $2dr$ number of measurements to recover $\X^\star$ from noiseless measurements (due to the null-space of the linear map). To test the various approaches, we vary the number of measurements between $2dr$ and $5dr$. We assume $r$ is known, though other computational experience suggests that estimates of $r$ return good answers as well.

The convex problem~\eqref{eq:conventional} depends on a parameter $\lambda$. We solve the problem for different $\lambda$ in a bracketing search until we find the first $\lambda$ that provides a solution with numerical rank $r$. 
Like~\cite{flammia2012quantum}, 
we normalize the final estimate to ensure the trace is $1$. 
Additionally, we test the following convex approach:
\begin{equation} \label{eq:conventional_2}
    \minimize_{ \X \succeq 0, \|\X\|_* \le 1} \|\linmap(\X) - \obs\|_F^2.
\end{equation}
Compared to~\eqref{eq:conventional}, no parameters are needed since we exploit prior knowledge of the trace, but there is no guarantee on the rank.  Both convex approaches can be solved with proximal gradient descent; we use the TFOCS package~\cite{becker2010templates} since it uses a sophisticated line search and Nesterov acceleration.

To illustrate the power of the combinatorial projections, we solve the following non-convex formulation:
\begin{equation} \label{eq:quantum_nonconvex}
    \minimize_{ \X \succeq 0, \|\X\|_* \le 1, \text{rank}(\X)=r} \|\linmap(\X) - \obs\|_F^2.
\end{equation}
Within the projected gradient algorithm~\eqref{eq: projected gradient}, we use the GSSP algorithm as described above. The stepsize is $\mu^i=3/\|\linmap\|^2$ where $\|\cdot\|$ is the operator norm; we can also apply Nesterov acceleration to speed convergence, but we use~\eqref{eq: projected gradient} for simplicity. Due to the non-convexity, the algorithm depends on the starting value $\X_0$. We try two strategies: $(i)$ random initialization, and $(ii)$ initializing with the solution from~\eqref{eq:conventional_2}. Both initializations often lead to the same stationary point.

Figure~\ref{fig:quantum1} shows the relative error $\|\X-\X^\star\|_F/\|\X^\star\|_F$ of the different approaches. All approaches do poorly when there are only $2dr$ measurements since this is near the noiseless information-theoretic limit. For higher numbers of measurements, the non-convex approach substantially outperforms both convex approaches. For $2.4dr$ measurements, it helps to start $\X_0$ with the convex solution, but otherwise the two non-convex approaches are nearly identical.

Between the two convex solutions, \eqref{eq:conventional} outperforms \eqref{eq:conventional_2} since it tunes $\lambda$ to achieve a rank $r$ solution. Neither convex approach is competitive with the non-convex approaches since they do not take advantage of the prior knowledge on trace and rank.  

Figure~\ref{fig:quantum2} shows more results on a $7$ qubit problem without noise. Again, the non-convex approach gives better results, particularly when there are fewer measurements. As expected, both approaches approach perfect recovery as the number of measurements increases.

\begin{table}
    \centering
    \begin{tabular}{lcc}
        \toprule 
        Approach & mean time & standard deviation \\
        \midrule
        convex & $0.294$~s. & $0.030$~s. \\
        non-convex & $0.192$~s. & $0.019$~s. \\
        \bottomrule
    \end{tabular}
    \caption{Time per iteration of convex and non-convex approaches for quantum state tomography with $8$ qubits.}
    \label{tab:quantum}
\end{table}

Here we highlight another key benefit of the non-convex approach: since the number of eigenvectors needed in the partial eigenvalue decomposition is at most $r$, it is quite scalable. In general, the convex approach has intermediate iterates which require eigenvalue decompositions close to the full dimension, especially during the first few iterations. Table~\ref{tab:quantum} shows average time per iteration for the convex and non-convex approach (overall time is more complicated, since the number of iterations depends on linesearch and other factors). Even using Matlab's dense eigenvalue solver \verb@eig@, the iterations of the non-convex approach are faster; problems that used an iterative Krylov subspace solver would show an even larger discrepancy.\footnote{Quantum state tomography does not easily benefit from iterative eigenvalue solvers, since the range of $\linmap^*$ is not sparse.}

\section{Application: Measure learning}
\begin{figure*}
   \begin{minipage}[t]{0.33\textwidth}
      \vspace{0pt}
      \includegraphics[width=\linewidth]{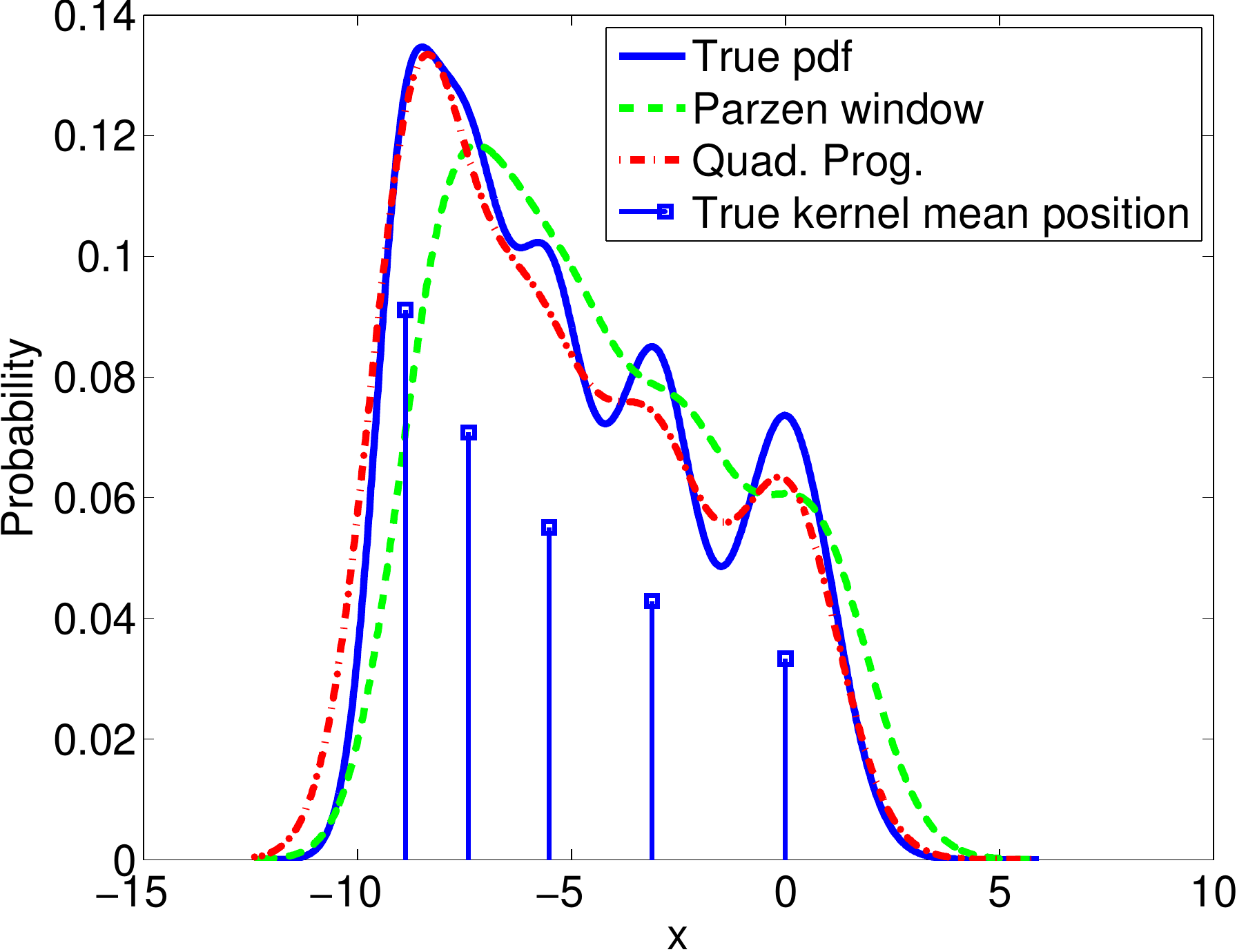}
   \end{minipage}
   \begin{minipage}[t]{0.33\textwidth}
      \vspace{0pt}
      \includegraphics[width=\linewidth]{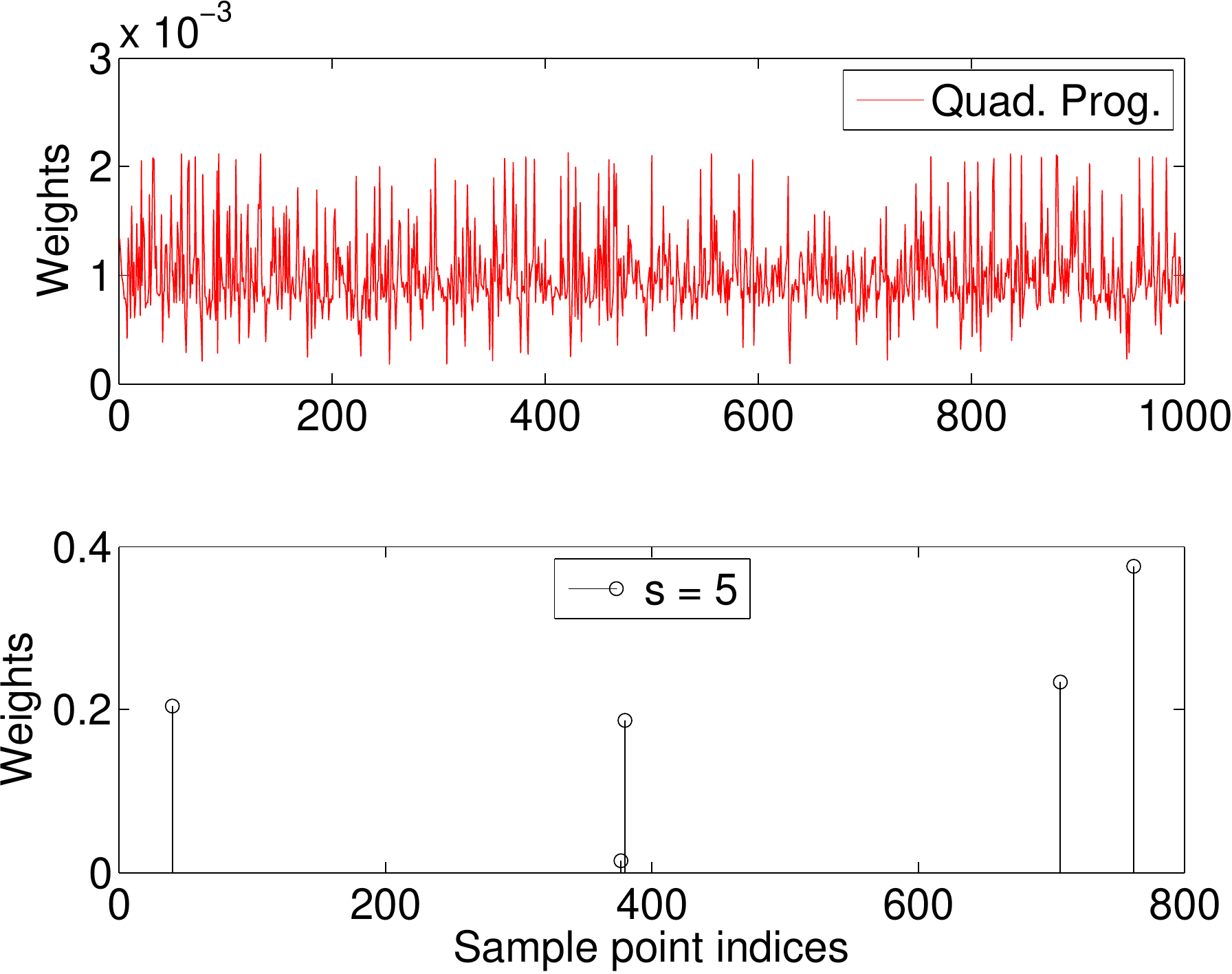}
   \end{minipage}
   \begin{minipage}[t]{0.33\textwidth}
      \vspace{0pt}
      \includegraphics[width=\linewidth]{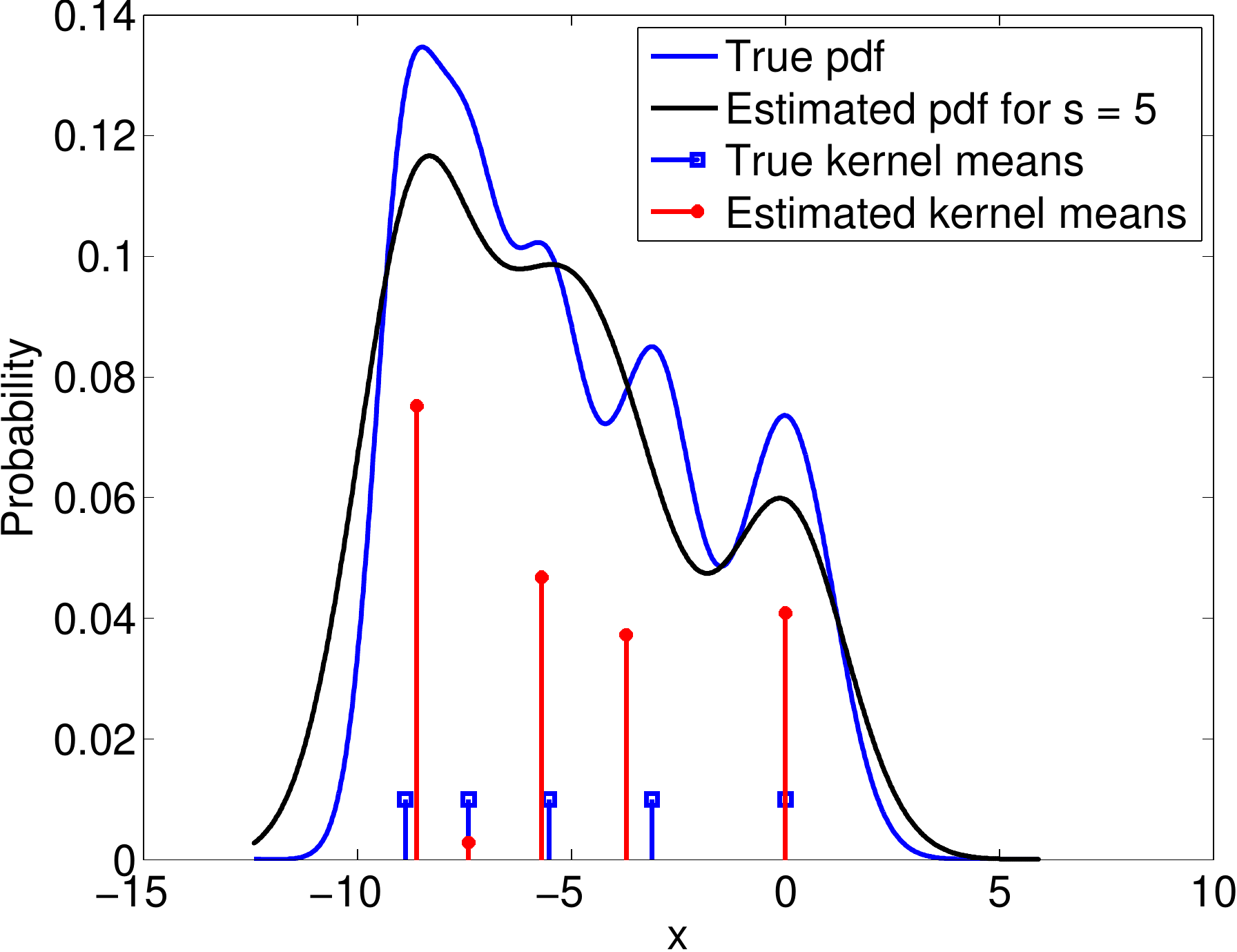}
      \end{minipage}\vspace{-.8em}
   \caption{\small\sl
Density estimation results using the Parzen method (left), the quadratic program \eqref{eq:SDE} (left and middle-top), and our approach (middle-bottom and right).} \label{fig: Exp6}
\end{figure*}

\begin{figure*}
   \begin{minipage}[t]{0.24\textwidth}
      \vspace{0pt}
      \subfigure[$\sparsity = 3$]{\includegraphics[width=\linewidth]{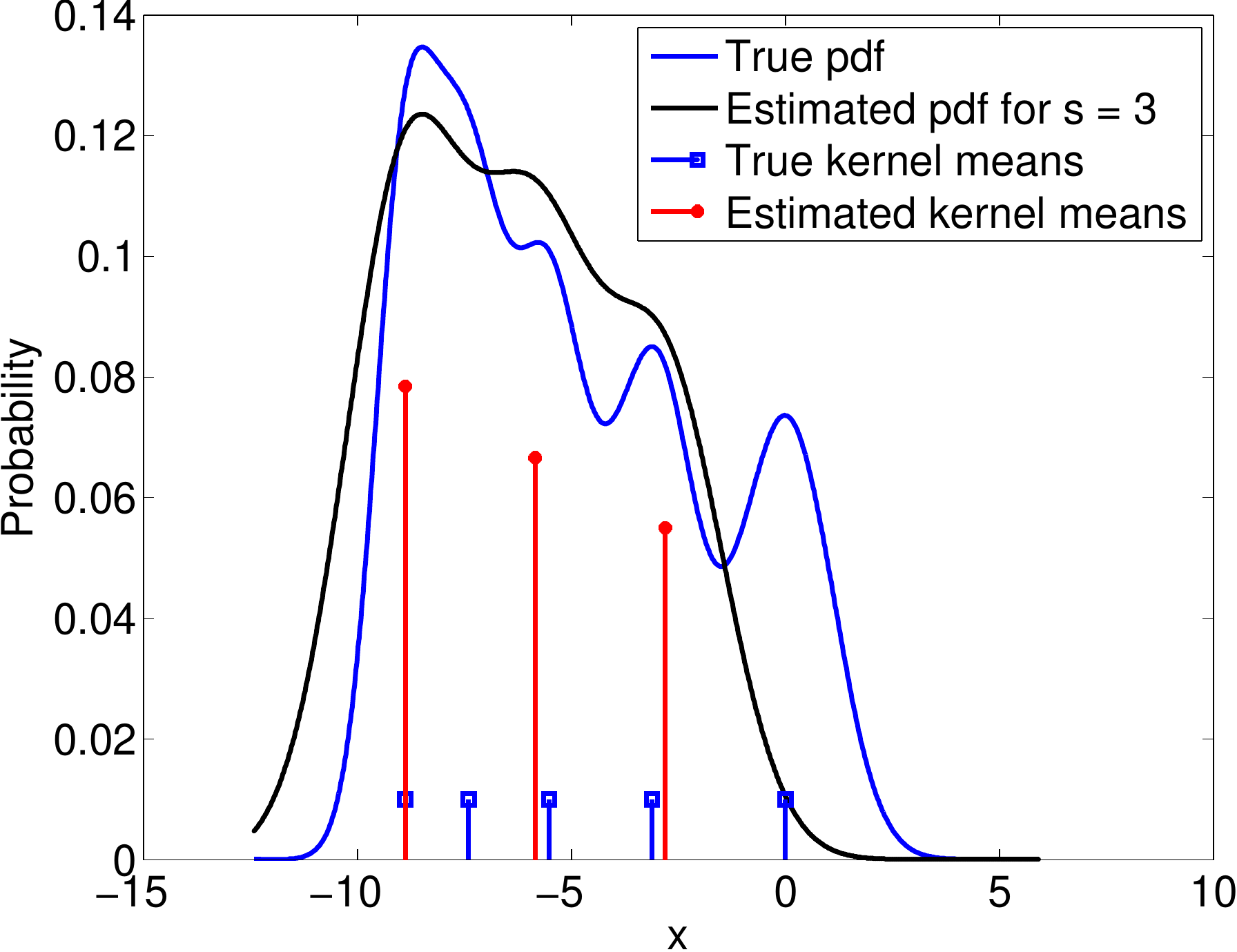}}
   \end{minipage}
   \begin{minipage}[t]{0.24\textwidth}
      \vspace{0pt}
      \subfigure[$\sparsity = 8$]{\includegraphics[width=\linewidth]{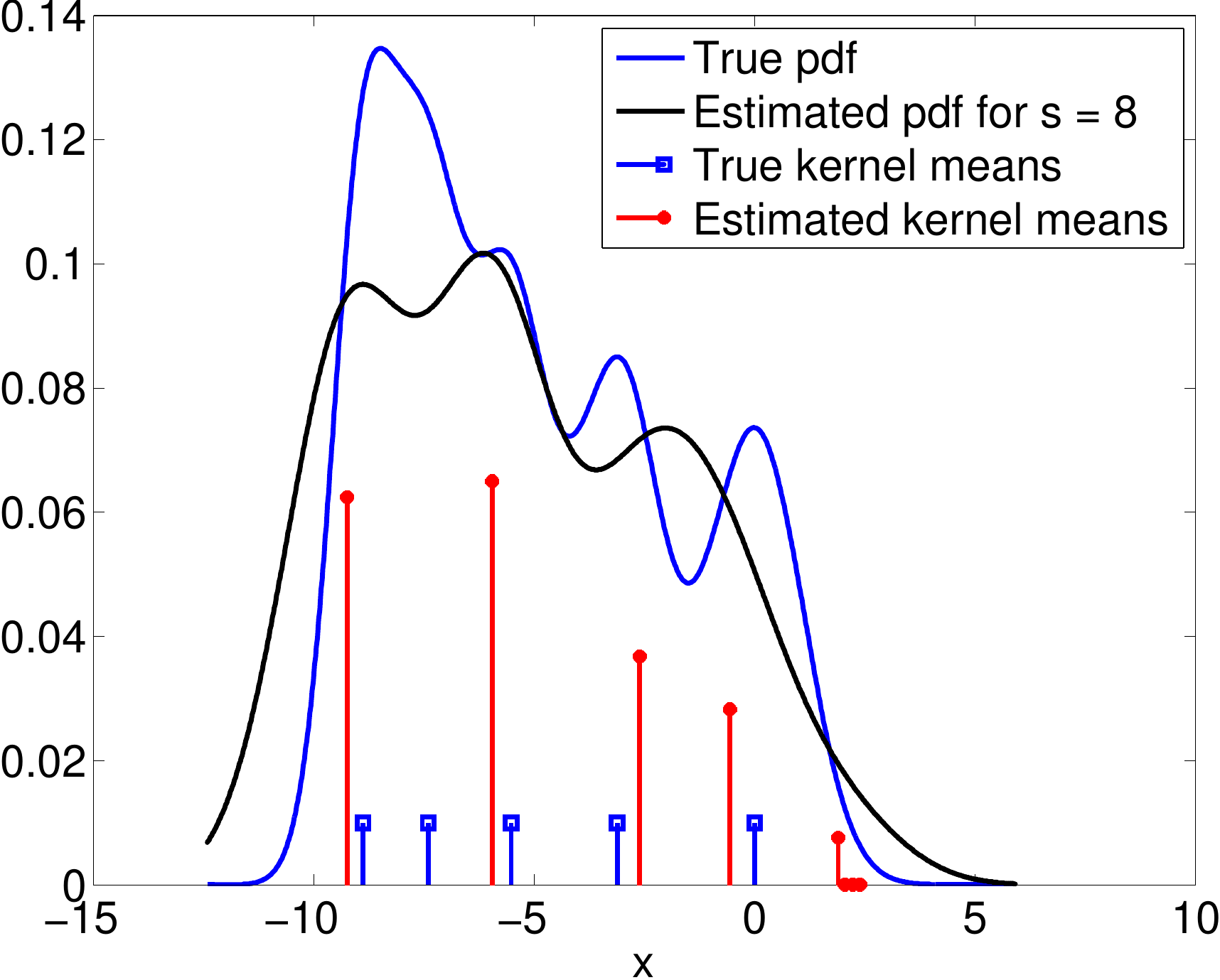}}
   \end{minipage}
   \begin{minipage}[t]{0.24\textwidth}
      \vspace{0pt}
      \subfigure[$\sparsity = 10$]{\includegraphics[width=\linewidth]{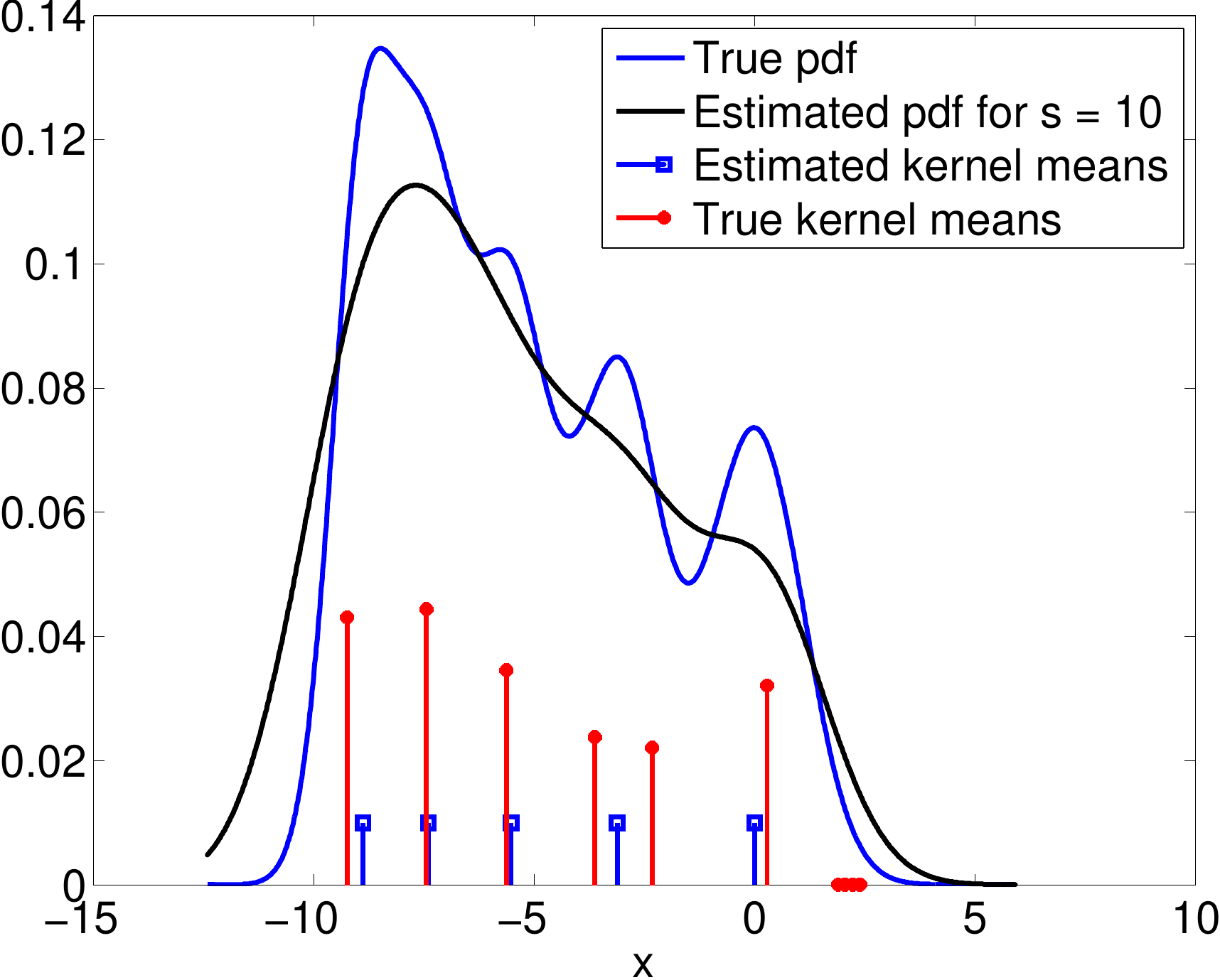}}
   \end{minipage}
   \begin{minipage}[t]{0.24\textwidth}
      \vspace{0pt}
      \subfigure[$\sparsity = 15$]{\includegraphics[width=\linewidth]{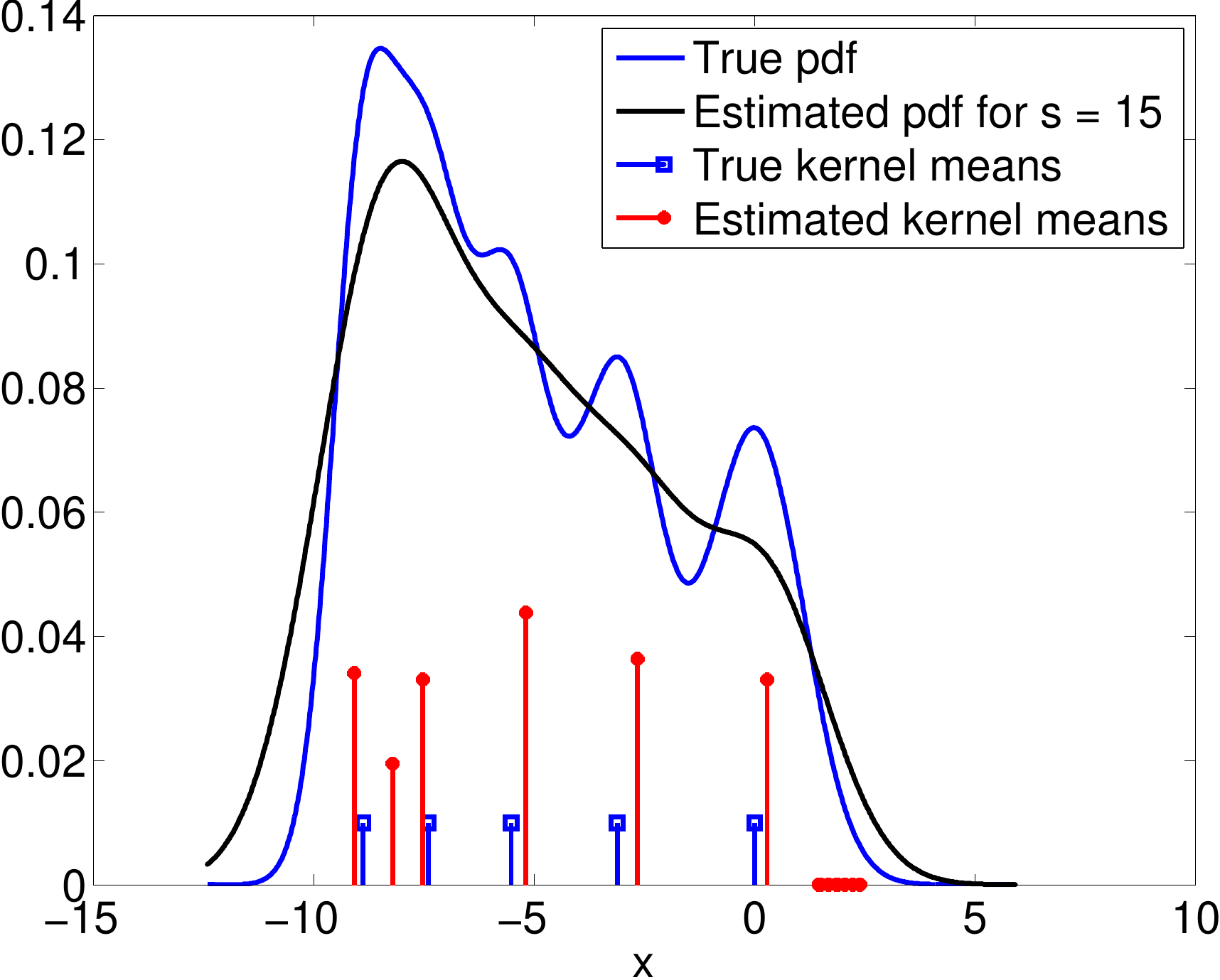}}
      \end{minipage} \vspace{-.3cm}
   \caption{\small\sl
Estimation results for different $\sparsity$: Red spikes depict the estimated kernel means as well as the their relative contribution to the Gaussian mixture. As $\sparsity$ increases, the additional nonzero coefficients in $\bestsignal$ tend to have small weights.} \label{fig: Exp7}
\vspace{-.3cm}
\end{figure*}

\newcommand{\pdf}{\mu}
\paragraph{Problem:} We study the kernel density learning setting: Let ${\bf x}^{(1)}, {\bf x}^{(2)},\dots, {\bf x}^{(n)} \in \mathbb{R}^{\dimension}$ be an $n$-size corpus of $\dimension$-dimensional samples, drawn from an unknown probability density function (pdf) $\pdf({\bf x})$. 
Here, we will form an estimator 
$\hat{\pdf}({\bf x}) := \sum_{i=1}^n \entry_i 
\kappa_{\sigma}({\bf x}, {\bf x}^{(i)}) $, 
where $\kappa_{\sigma}({\bf x}, {\bf y}) $ is a Gaussian kernel with parameter $\sigma$. 
Let us choose $\hat{\pdf}({\bf x})$ to minimize the integrated squared error criterion: $\text{ISE} = \mathbb{E}\vectornorm{\hat{\pdf}({\bf x}) - \pdf({\bf x})}_2^2$. As a result, we can introduce a density learning problem as estimating a weight vector $\bestsignal \in \Delta_{1}^{+}$. The objective can then be written as follows \cite{kim1995least,bunea2010spades}
\begin{align}
\bestsignal \in \argmin_{\signal \in \Delta_1^{+}} \left\{ \signal^T \boldsymbol{\Sigma} \signal - {\bf c}^T \signal\right\}, \label{eq:SDE}
\end{align}
where $\boldsymbol{\Sigma} \in \R^{n \times n}$ with $\Sigma_{ij} = \kappa_{\sqrt{2}\sigma}({\bf x}^{(i)}, {\bf x}^{(j)})$, and \begin{align}
c_i = \frac{1}{n - 1}\sum_{j \neq i}\kappa_{\sigma}({\bf x}^{(i)}, {\bf x}^{(j)}), ~\forall i, j.
\end{align}

While the combination of the $-{\bf c}^T \signal$ term and the non-negativity constraint induces some sparsity, it may not be enough. 
To avoid overfitting or obtain interpretable results, one might control the level of solution sparsity \cite{bunea2010spades}. In this context, we extend \eqref{eq:SDE} to include cardinality constraints, i.e. $\bestsignal \in \Delta_{1}^{+} \cap \constraint$. 


\paragraph{Numerical experiments: }
We consider the following Gaussian mixture: $\pdf(x) = \frac{1}{5}\sum_{i = 1}^5 \kappa_{\sigma_i}(\signal_i, x)$, where $\sigma_i = (7/9)^i$ and $\signal_i = 14(\sigma_i - 1)$. A sample of $1000$ points is drawn from $\pdf(x)$. We compare the density estimation performance of: $(i)$ the Parzen  method \cite{parzen1962estimation}, $(ii)$ the quadratic programming formulation in \eqref{eq:SDE}, and $(iii)$ our cardinality-constrained version of \eqref{eq:SDE} using GSSP. While $\pdf(x)$ is constructed by kernels with various widths, we assume a constant width during the kernel estimation. In practice, the width is not known {\em a priori} but can be found using cross-validation techniques; for simplicity, we assume kernels with width $\sigma = 1$.

Figure \ref{fig: Exp6}(left) depicts the true pdf and the estimated densities using the Parzen  method and the quadratic programming approach. Moreover, the figure also includes a scaled plot of $1/\sigma_i$, indicating the height of the individual Gaussian mixtures. By default, the Parzen window method estimation interpolates $1000$ Gaussian kernels with centers around the sampled points to compute the estimate $\hat{\pdf}(x)$; unfortunately, neither the quadratic programming approach (as Figure \ref{fig: Exp6} (middle-top) illustrates) nor the Parzen window estimator results are easily {\it interpretable}  even though both approaches provide a good approximation of the true pdf.

Using our cardinality-constrained approach, we can significantly enhance the interpretability. This is because in the sparsity-constrained approach, we can control the number of estimated Gaussian components. Hence, if the model order is known \emph{a priori}, the non-convex approach can be extremely useful. 

To see this, we first show the coefficient profile of the sparsity based approach  for $\sparsity = 5$ in Figure \ref{fig: Exp6} (middle-bottom). Figure \ref{fig: Exp6} (right) shows the estimated pdf for $\sparsity = 5$ along with the positions of weight coefficients obtained by our sparsity enforcing approach. Note that most of the weights obtained concentrate around the true means, fully exploiting our prior information about the ingredients of $\pdf(x)$---this happens with rather high frequency in the experiments we conducted. Figure \ref{fig: Exp7} illustrates further estimated pdf's using our approach for various $\sparsity$. 
Surprisingly, the resulting solutions are still approximately $5$-sparse even if $k>5$, as the over-estimated coefficients are extremely small, and hence the sparse estimator is reasonably robust to inaccurate estimates of $k$.

\section{Application: Portfolio optimization} \label{sec:portfolio}
\paragraph{Problem:} Given a sample covariance matrix $\boldsymbol{\Sigma}$ and expected mean $\boldsymbol{\mu}$, the return-adjusted Markowitz mean-variance (MV) framework selects a portfolio $\bestsignal $ such  that
$
\bestsignal \in \argmin_{\signal \in \Delta_1^{+}} \left\{ \signal^T \boldsymbol{\Sigma} \signal - \tau\boldsymbol{\mu}^T \signal\right\},
$
where $\Delta_1^{+}$ encodes the normalized capital constraint, and $\tau$ trades off risk and return \cite{demiguel2009generalized,brodie2009sparse}. The solution $\bestsignal \in \Delta_1^{+}$ is the distribution of investments over the $\dimension$ available assets. 

\newcommand{\bsd}{\boldsymbol{\delta}_{\signal}}
While such solutions construct portfolios from scratch, a more realistic scenario  
is to  
incrementally adjust an existing portfolio as the market changes. Due to costs per transaction, we can naturally introduce cardinality constraints. 
In mathematical terms, let $\bar{\signal} \in \mathbb{R}^{\dimension}$ be the current portfolio selection. Given $\bar{\signal}$, we seek to adjust the current selection  $\signal = \bar{\signal} + \boldsymbol{\delta}_{\signal}$ such that $\vectornorm{\boldsymbol{\delta}_{\signal}}_0 \leq k$.  
This leads to the following optimization problem: 
\begin{align}\label{exp2:eq:1}\nonumber 
    \bsd^\ast \in
\argmin_{\bsd\in \Sigma_\sparsity \cap \Delta_\lambda} 
(\bar{\signal} + 
\bsd)^T \Sigma (\bar{\signal}+ \bsd )- \tau \mu^T (\bar{\signal}+ \bsd), 
\end{align} where $\lambda$ is the level of update, and $k$ controls the transactions costs. During an update, $\lambda=0$ would keep the portfolio value constant while $\lambda>0$ would  increase it. 

%
\paragraph{Numerical experiments:} To clearly highlight the impact of the non-convex projector, we create a synthetic portfolio update problem, where we know the solution. As in \cite{brodie2009sparse}, we cast this problem as a regression problem and synthetically generate $\obs = {\bf X} \bestsignal$ where $p=1000$ such that $\bestsignal \in \Delta_\lambda$ ($\lambda$ is chosen randomly), and $\vectornorm{\bestsignal}_0 = \sparsity$ for $\sparsity = 100$. 

Since in general we do not expect RIP assumptions to hold in portfolio optimization, our goal here is to refine the sparse solution of a state-of-the-art convex solver via \eqref{eq: projected gradient} in order to accommodate the strict sparsity and budget constraints. Hence, we first consider the basis pursuit criterion and solve it using SPGL1~\cite{SPGL}:
\begin{equation} \label{eq:BP}
    \minimize \vectornorm{\signal}_1~~\text{s.t.}~~\begin{bmatrix}
    {\bf X} \\ \ones^T/\sqrt{p} 
    \end{bmatrix}  
    \signal 
    = \begin{bmatrix} \obs \\  \lambda/\sqrt{p} \end{bmatrix}.
\end{equation}
The normalization by $1/\sqrt{\dimension}$ in the last equality gives the constraint matrix a better condition number, since otherwise it is too ill-conditioned for a first-order solver.

Almost none of the solutions to~\eqref{eq:BP} return a $\sparsity$-sparse solution. 
Hence, we initialize \eqref{eq: projected gradient} with the SPGL1 solution to meet the constraints. 
The update step in \eqref{eq: projected gradient} uses the GSHP algorithm.

Figure~\ref{fig:SparseHyper} shows the resulting relative errors $\vectornorm{\widehat{\signal} - \bestsignal}_2 / \vectornorm{\bestsignal}_2$. 
We see that not only does \eqref{eq: projected gradient} return a $\sparsity$-sparse solution, but that this solution is also closer to $\bestsignal$, particularly when the sample size is small.
As the sample size increases, the knowledge that $\bestsignal$ is $\sparsity$-sparse makes up a smaller percentage of what we know about the signal, so the gap between \eqref{eq:BP} and \eqref{eq: projected gradient} diminishes.
\begin{figure}[!t]
\centering
\includegraphics[width=.95\linewidth]{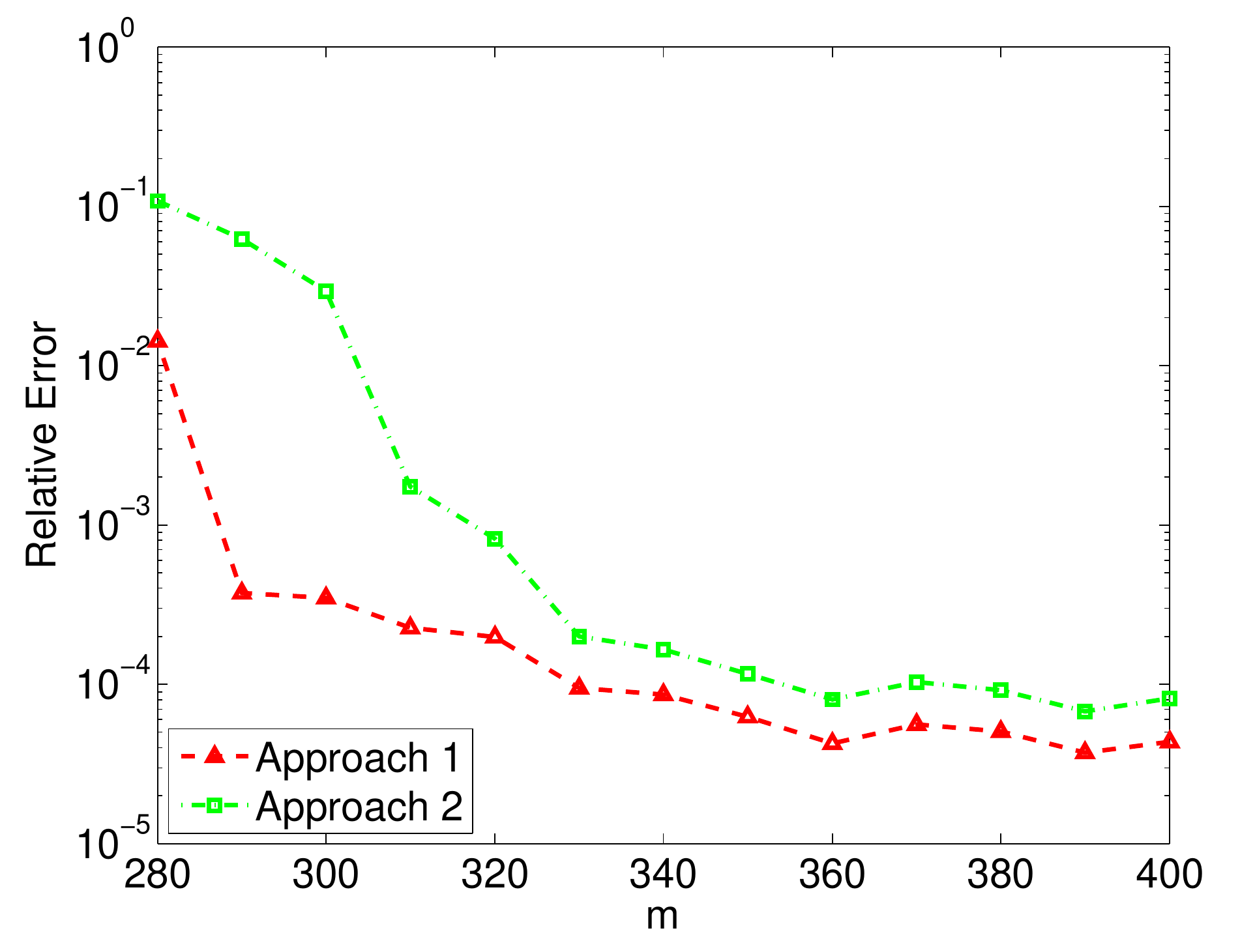}\vspace{-.15cm}
\caption{Relative error $\vectornorm{\widehat{\signal} - \bestsignal}_2 / \vectornorm{\bestsignal}_2$ comparison as a function of $\numsam$: Approach 1 is the non-convex approach \eqref{eq: projected gradient}, and approach 2 is \eqref{eq:BP}. Each point corresponds to the median value of 30 Monte-Carlo realizations.}
    \label{fig:SparseHyper}
\end{figure}

\section{Conclusions}\label{conclusion}
While non-convexity in learning algorithms is undesirable according to conventional wisdom, avoiding it might be difficult in many problems. In this setting, we show how to efficiently obtain exact sparse projections onto the positive simplex and its hyperplane extension. We empirically demonstrate that our projectors provide substantial accuracy benefits in quantum tomography from fewer measurements and enable us to exploit prior non-convex knowledge in density learning. Moreover, we also illustrate that we can refine the solution of well-established state-of-the-art convex sparse recovery algorithms to enforce non-convex constraints in sparse portfolio updates. The quantum tomography example in particular illustrates that the non-convex solutions can be extremely useful; here, the non-convexity appears milder, 
since a fixed-rank matrix still has extra degrees of freedom from the choice of its eigenvectors.
%

\section{Acknowledgements}
VC and AK's work was supported in part by the European Commission under Grant MIRG-268398, ERC Future Proof and SNF 200021-132548.
SRB is supported by the Fondation Sciences Math\'ematiques de Paris. CK's research is funded by ERC ALGILE.
\bibliography{recipes,mindexmodels,snf-QT-biblio}
\bibliographystyle{icml2013}

\end{document}